\documentclass{elsarticle}
\usepackage{epsfig}
\usepackage{amssymb}
\usepackage{amsmath}
\usepackage{times}
\usepackage{indentfirst}
\usepackage{graphicx}
\usepackage{algorithm}
\usepackage{algorithmic}
\usepackage{multirow}
\newtheorem{theorem}{Theorem}
\newtheorem{corollary}{Corollary}
\newtheorem{definition}{Definition}
\newtheorem{proposition}{Proposition}
\newtheorem{example}{Example}

\newdefinition{remark}{Remark}
\newproof{proof}{Proof}

\begin{document}
\begin{frontmatter}

\title{Characteristic matrix of covering and its application to boolean matrix decomposition}
\author[addr1]{Shiping Wang}
\author[addr2]{William Zhu\corref{cor1}}\ead{williamfengzhu@gmail.com}
\author[addr1]{Qingxin Zhu}
\author[addr2]{Fan Min}

\cortext[cor1]{Corresponding author.}
\address[addr1]{School of Computer Science and Engineering,\\University of Electronic Science and Technology of China, Chengdu 611731, China}
\address[addr2]{Lab of Granular Computing,\\Zhangzhou Normal University, Zhangzhou 363000, China}


\date{\today}
\begin{abstract}
Covering-based rough sets provide an efficient theory to deal with covering data which widely exist in practical applications.
Boolean matrix decomposition has been widely applied to data mining and machine learning.
In this paper, three types of existing covering approximation operators are represented by boolean matrices, and then they are used to decompose into boolean matrices.
First, we define two types of characteristic matrices of a covering.
Through these boolean characteristic matrices, three types of existing covering approximation operators are concisely equivalently represented.
Second, these operators are applied to boolean matrix decomposition.
We provide a sufficient and necessary condition for a square boolean matrix to decompose into the boolean product of another one and its transpose.
Then we develop an algorithm for this boolean matrix decomposition.
Finally, these three types of covering approximation operators are axiomatized using boolean matrices.
In a word, this work presents an interesting view to investigate covering-based rough set theory and its application.
\end{abstract}
\end{frontmatter}

\section{Introduction}

As a widely used form of data representation, coverings commonly appear in incomplete information/decision systems based on symbol data~\cite{BianucciCattaneoCiucci07Entropies,CousoDubois01Rough,QianLiangDang10Incomplete}, numeric and fuzzy data~\cite{HuYuXie08Numerical,HuZhangPanAn11Measuring,FanZhu12Attribute}.
Covering-based rough set theory~\cite{Pomykala87Approximation,Zakowski83Approximations} is an efficient tool to deal with covering data.
In recent years, it has attracted much research interest with a series of significant problems proposed.
For example, different approximation models have been constructed~\cite{DaiXu12Approximation,WangZhu12Quantitative,YaoYao12Covering,ZhuWang07OnThree}, covering reduction problems have been defined~\cite{DuHuZhuMa11Rule,MinHeQianZhu11Test,WangZhuZhuMin12Covering,ZhuWang03Reduction}, generalization works have been conducted~\cite{DengChenXuDai07ANovel,FengZhangMi12TheRelation,LiLeungZhang08Generalized,YamakKazancDavvaz10Generalized,Zhang12Generalized} and combinations with other theories have been made~\cite{GeBaiYun12Topological,GhanimTantawySelim97OnLower,LiuZhu08TheAlgebraic,WangZhuZhuMin12Matroidal,YamakKazanciDavvaz08Applications}.

Specifically, axiomatization of covering approximation operators has become a hot issue.
For example, Zhu and Wang~\cite{ZhuWang03Reduction,ZhuWang07OnThree} proposed the reducible element to axiomatize the covering lower approximation operator.
Following by Zhu and Wang's work, Zhang et al.~\cite{ZhangLiWu10OnAxiomatic} axiomatized three pairs of covering approximation operators.
Liu and Sai~\cite{LiuSai09AComparison} constructed an axiom of a pair of covering approximation operators from the viewpoint of operator theory.
Unfortunately, as an efficient tool for computational models, matrices are seldom used in covering-based rough sets.
However, more and more achievements have been made in representing and axiomatizing classical and fuzzy rough sets using matrices~\cite{Liu08Generalized,Liu10Rough,LiuSai10Invertible}.
Naturally, this motivates us to represent and axiomatize covering-based rough sets using boolean matrices.

Boolean matrix decomposition has not only important practical meaning, but also profound theoretical significance.
In application, it has been widely used in data mining~\cite{BelohlavekVychodil10Discovery}, role engineering~\cite{LuVaidyaAtluri08Optimal} and machine learning~\cite{FrankAndreas12Multi-assignment}, and so on.
In theory, boolean matrix decomposition problems, such as deterministic column-based matrix decomposition and boolean matrix factorization, have attracted much research interest~\cite{AkelbekFitalShen09ABound,LiPang10Deterministic,Lingas11AFast}.
The boolean matrix representations of covering approximation operators also motivate us to decompose into boolean matrices from the viewpoint of rough set.

In this paper, we represent three pairs of covering approximation operators using boolean matrices and the representations in turn are used to decompose into boolean matrix.
First, we define two types of characteristic matrices of a covering and use them to concisely represent three pairs of covering approximation operators.
Second, through the matrix representation of covering upper approximation operator, we present a sufficient and necessary condition for a square boolean matrix to decompose into the boolean product of another boolean matrix and its transpose.
And an algorithm to complement this decomposition is designed.
Third, using the sufficient and necessary condition of boolean matrix decomposition, we axiomatize these three types of covering approximation operators.

The rest of this paper is arranged as follows.
Section~\ref{S:BasicDefinitions} reviews some fundamental concepts related to covering-based rough sets.
In Section~\ref{S:Matrixrepresentationofcoveringapproximationoperators}, we present two types of characteristic matrices of a covering and use them to represent three types of covering approximation operators.
Section~\ref{S:booleanmatrixdecompositionusingcoveringbasedroughsets} exhibits a sufficient and necessary condition, and design an algorithm for boolean matrix decomposition.
In Section~\ref{S:Axiomatizationofcovering-basedroughsetsusingmatrices}, we axiomatize these three types of covering approximation operators through boolean matrix decomposition.
Section~\ref{S:Conclusions} concludes this paper and points out further work.

\section{Basic definitions}
\label{S:BasicDefinitions}

This section recalls some fundamental definitions and existing results concerning covering-based rough sets.

\begin{definition}(Covering~\cite{ZhuWang03Reduction})
Let $U$ be a finite universe of discourse and $\mathbf{C}$ a family of subsets of $U$.
If none of subsets in $\mathbf{C}$ is empty and $\bigcup \mathbf{C}=U$, then $\mathbf{C}$ is called a covering of $U$.
\end{definition}

Neighborhoods are important concepts in rough sets, and they characterize the maximal and minimal dependence between an object and others.

\begin{definition}(Indiscernible neighborhood and neighborhood~\cite{Zhu09RelationshipBetween})
Let $\mathbf{C}$ be a covering of $U$ and $x\in U$.
$I_{\mathbf{C}}(x)=\bigcup\{K\in\mathbf{C}|x\in K\}$ and $N_{\mathbf{C}}(x)=\bigcap\{K\in\mathbf{C}|x\in K\}$ are called the indiscernible neighborhood and neighborhood of $x$ with respect to $\mathbf{C}$, respectively.
When there is no confusion, we omit the subscript $\mathbf{C}$.
\end{definition}

Neighborhood granule derived from coverings is a basic unit to characterize data, and leads to neighborhood-based decision systems, where neighborhood-based approximation operators have been used extensively in symbolic or/and numerical attribute reduction~\cite{HuYuXie08Numerical}.
In this paper, we study the following three types of lower and upper approximation operators.

\begin{definition}(Approximation operators~\cite{Pomykala87Approximation,Zhu09RelationshipBetween})
Let $\mathbf{C}$ be a covering of $U$. For all $X\subseteq U$, \\
$SH_{\mathbf{C}}(X)=\bigcup\{K\in\mathbf{C}|K\bigcap X\neq\emptyset\}$, $SL_{\mathbf{C}}(X)= [SH_{\mathbf{C}}(X^{c})]^{c}$,\\
$IH_{\mathbf{C}}(X)=\{x\in U|N(x)\bigcap X\neq \emptyset\}$, $IL_{\mathbf{C}}(X)=\{x\in U|N(x)\subseteq X\}$, \\
$XH_{\mathbf{C}}(X)=\bigcup\{N(x)|N(x)\bigcap X\neq \emptyset\}$, $XL_{\mathbf{C}}(X)=\bigcup\{N(x)|N(x)\subseteq X\}$, \\
are called the second, fifth, and sixth upper and lower approximations of $X$ with respect to $\mathbf{C}$, respectively.
When there is no confusion, we omit the subscript $\mathbf{C}$.
\end{definition}

In practical applications, much knowledge is redundant, therefore it is necessary to remove the redundancy and keep the essence.
For example, the reducible element has been applied to knowledge redundancy in rule learning~\cite{DuHuZhuMa11Rule}.

\begin{definition}(Reducible element~\cite{ZhuWang03Reduction})
Let $\mathbf{C}$ be a covering of $U$ and $K\in\mathbf{C}$.
If $K$ is a union of some sets in $\mathbf{C}-\{K\}$, then $K$ is called reducible; otherwise $K$ is called irreducible.
The family of all irreducible elements of $\mathbf{C}$ is called the reduct of $\mathbf{C}$, denoted as $Reduct(\mathbf{C})$.
\end{definition}

\section{Matrix representation of covering approximation operators}
\label{S:Matrixrepresentationofcoveringapproximationoperators}

In this section, we define the matrix representation of a family of subsets of a set, and then propose two types of characteristic matrices of a covering~\cite{WangZhuZhuMin12Characteristic}.
Through these two characteristic matrices of a covering, we represent three types of covering approximation operators.

\subsection{Matrix representation of a family of subsets of a set}
\label{S:matrixexpressionsoffamilyofsubsets}

This subsection represents a family of subsets of a set using a zero-one matrix, called boolean matrix.
Using this matrix, families of subsets are connected with binary relations and their further properties are found.
For any $n\times m$ matrix $M$, we denote $M = (M_{ij})_{n\times m}$ unless otherwise stated.

\begin{definition}
Let $\mathbf{F}=\{F_{1}, \cdots, F_{m}\}$ be a family of subsets of a finite set $U=\{x_{1}, \cdots, x_{n}\}$.
We define $M_{\mathbf{F}}=((M_{\mathbf{F}})_{ij})_{n\times m}$ as follows:
\begin{center}
$(M_{\mathbf{F}})_{ij}=\left\{
\begin{matrix}
1, &x_{i}\in F_{j},\\
0, &x_{i}\notin F_{j}.
\end{matrix}\right.$
\end{center}
$M_{\mathbf{F}}$ is called a matrix representation of $\mathbf{F}$, or called a matrix representing $\mathbf{F}$.
\end{definition}

The following example shows that different matrices can be used to represent the same family of subsets of a set.

\begin{example}
\label{E:matricesrepresenting}
Let $U=\{a, b, c, d, e\}$ and $\mathbf{F}=\{\{a, b, c\}, \{b, d\}, \{c, d\}\}$.
Then $M_{1}$ and $M_{2}$ are matrices representing $\mathbf{F}$.
\begin{center}
$M_{1}=\left(\begin{array}{lcr}
  1 & 0 & 0\\
  1 & 1 & 0\\
  1 & 0 & 1\\
  0 & 1 & 1\\
  0 & 0 & 0
\end{array}\right)$,
$M_{2}=\left(\begin{array}{lcr}
  0 & 1 & 0\\
  1 & 1 & 0\\
  0 & 1 & 1\\
  1 & 0 & 1\\
  0 & 0 & 0
\end{array}\right)$.
\end{center}
\end{example}

There are matrices representing the same family of subsets of a set, however it is interesting that the boolean product of one matrix and its transpose is unique when an order of elements of the universe is given.

\begin{proposition}
Let $\mathbf{F}=\{F_{1}, \cdots, F_{m}\}$ be a family of subsets of $U=\{x_{1}, \cdots, x_{n}\}$ and $M_{1}$, $M_{2}$ matrices representing $\mathbf{F}$.
Then $M_{1}\cdot M_{1}^{T}=M_{2}\cdot M_{2}^{T}$, where $M\cdot M^{T}$ is the boolean product of $M$ and its transpose $M^{T}$.
\end{proposition}
\begin{proof}
Since $M_{1}$ and $M_{2}$ are matrices representing $\mathbf{F}$, $M_{1}$ can be transformed into $M_{2}$ through list exchanges.
Hence we only need to prove $M_{1}\cdot M_{1}^{T}=M_{2}\cdot M_{2}^{T}$ when $M_{1}=(a_{1}, \cdots, a_{i}, \cdots, a_{j}, \cdots, a_{m})$ and $M_{2}=(a_{1}, \cdots, a_{j}, \cdots, a_{i}, \cdots, a_{m})$ for $1\leq i < j\leq m$, where $a_{k}$ ia a $n$-dimensional column vector.
Thus $M_{1}\cdot M_{1}^{T}=(a_{1}, \cdots, a_{m})\cdot$
$\left(\begin{array}{c}
a_{1}^{T}\\
\vdots\\
a_{m}^{T}
\end{array}\right)$=$\vee_{k=1}^{m}(a_{k}\cdot a_{k}^{T})=M_{2}\cdot M_{2}^{T}$.
\end{proof}

The following example is provided to illustrate the uniqueness of the boolean product of any matrix representing a covering and its transpose.

\begin{example}
As shown in Example~\ref{E:matricesrepresenting}, $M_{1}\cdot M_{1}^{T}=M_{2}\cdot M_{2}^{T}=\left(
\begin{matrix}
1 & 1 & 1 & 0 & 0\\
1 & 1 & 1 & 1 & 0\\
1 & 1 & 1 & 1 & 0\\
0 & 1 & 1 & 1 & 0\\
0 & 0 & 0 & 0 & 0
\end{matrix}\right).$
\end{example}

Suppose $R$ is a relation on $U=\{x_{1}, \cdots, x_{n}\}$, then its relational matrix  $M_{R}=((M_{R})_{ij})_{n\times n}$ is defined as follows: \\
\begin{center}
$(M_{R})_{ij}=\left\{
\begin{matrix}
1, ~~(x_{i}, x_{j})\in R, \\
0, ~~(x_{i}, x_{j})\notin R.
\end{matrix}\right.$
\end{center}
Conversely, for any $n$-by-$n$ boolean matrix $M$, there exists a relation $R$ such that $M=M_{R}$; we say $R$ is induced by $M$.

According to the above notation, there is a one-to-one correspondence between binary relations on $U$ and $|U|\times |U|$ boolean matrices, which builds the connection between matrices representing families of subsets and binary relations.

\begin{proposition}
Let $\mathbf{F}$ be a family of subsets of $U$.
There exists a symmetric relation $R$ such that $M_{\mathbf{F}}\cdot M_{\mathbf{F}}^{T}$ is the relational matrix of $R$.
\end{proposition}
\begin{proof}
We need to prove only that $M_{\mathbf{F}}\cdot M_{\mathbf{F}}^{T}$ is a symmetric matrix.
It is straightforward since $(M_{\mathbf{F}}\cdot M_{\mathbf{F}}^{T})^{T}=(M_{\mathbf{F}}^{T})^{T}\cdot M_{\mathbf{F}}^{T}=M_{\mathbf{F}}\cdot M_{\mathbf{F}}^{T}$.
\end{proof}

The square matrix, the boolean product of a matrix representing a family of subsets of a universe and its transpose, is regarded as a whole and satisfies idempotence.

\begin{proposition}
Let $\mathbf{F}=\{F_{1}, \cdots, F_{m}\}$ be a family of subsets of $U$.
If for all $1\leq i<j\leq m$, $F_{i}\bigcap F_{j}=\emptyset$, then $(M_{\mathbf{F}}\cdot M_{\mathbf{F}}^{T})^{2}=M_{\mathbf{F}}\cdot M_{\mathbf{F}}^{T}$.
\end{proposition}
\begin{proof}
Denote $M_{\mathbf{F}}\cdot M_{\mathbf{F}}^{T}=\left(
\begin{matrix}
a_{1}\\
\vdots\\
a_{n}
\end{matrix}\right)\cdot(a_{1}^{T}, \cdots, a_{n}^{T})=(t_{ij})_{n\times n}$ and $(M_{\mathbf{F}}\cdot M_{\mathbf{F}}^{T})^{2}=(s_{ij})_{n\times n}$.
For all $1\leq i< j\leq m$, $F_{i}\bigcap F_{j}=\emptyset$, then \\
$t_{ij}=a_{i}\cdot a_{j}^{T}=\left\{
\begin{array}{lcr}
1, ~\exists F\in\mathbf{F}$, s.t. $x_{i}, x_{j}\in F,\\
0$,~otherwise.
$\end{array}\right.$.
If $t_{ij}=1$, then $s_{ij}=\vee_{k=1}^{n}(t_{ik}\wedge t_{kj})\geq t_{ij}\wedge t_{jj}=1$, which implies $s_{ij}=1$.
If $t_{ij}=0$ and $i=j$, then $x_{i}\notin\bigcup\mathbf F$ and $a_{i}=[0, \cdots, 0]$, which implies $s_{ij}=0$.
If $t_{ij}=0$ and $i\neq j$, then we need to prove $s_{ij}=0$.
In fact, if $s_{ij}=\vee_{k=1}^{n}(t_{ik}\wedge t_{kj})=1$, then there exists $k_{0}\in\{1, \cdots, n\}$ such that $t_{ik_{0}}=t_{jk_{0}}=1$.
Thus there exist $F_{g}, F_{h}\in\mathbf{F}$ such that $x_{i}, x_{k_{0}}\in F_{g}$ and $x_{j}, x_{k_{0}}\in F_{h}$.
Since $t_{ij}=0$, $F_{g}\neq F_{h}$.
Therefore, $x_{k_{0}}\in F_{g}\bigcap F_{h}$, which is contradictory with $F_{i}\bigcap F_{j}=\emptyset$ for all $i, j\in\{1, \cdots, n\}$ and $i\neq j$.
\end{proof}

\begin{corollary}
Let $\mathbf{F}=\{F_{1}, \cdots, F_{m}\}$ be a family of subsets of $U$.
If for all $1\leq i<j\leq m$, $F_{i}\bigcap F_{j}=\emptyset$, then there exists a transitive relation $R$ such that $M_{\mathbf{F}}\cdot M_{\mathbf{F}}^{T}$ is the relational matrix of $R$.
\end{corollary}

\begin{corollary}
If $\mathbf{P}$ is a partition of $U$, then $(M_{\mathbf{P}}\cdot M_{\mathbf{P}}^{T})^{2}=M_{\mathbf{P}}\cdot M_{\mathbf{P}}^{T}$.
\end{corollary}
\subsection{Type-1 characteristic matrix of covering}
\label{S:matrixrepresentingofcovering}

In this subsection, type-1 characteristic matrix of a covering is defined, and then connections between coverings and binary relations are established.

\begin{definition}(Type-1 characteristic matrix of covering)
Let $\mathbf{C}$ be a covering of $U$.
Then $M_{\mathbf{C}}\cdot M_{\mathbf{C}}^{T}$ is called type-1 characteristic matrix of $\mathbf{C}$, denoted as $\Gamma(\mathbf{C})$.
\end{definition}

Properties of type-1 characteristic matrix of a covering are studied.
In fact, its elements on the main diagonal are equal to one.

\begin{proposition}
\label{P:maindiagonalequaltoone}
Let $\mathbf{F}$ be a family of subsets of $U$ and $\emptyset\notin\mathbf{F}$.
$\mathbf{F}$ is a covering iff all the elements on the main diagonal of $M_{\mathbf{F}}\cdot M_{\mathbf{F}}^{T}$ are one.
\end{proposition}
\begin{proof}
Denote $M_{\mathbf{F}}=\left(
\begin{matrix}
a_{1}\\
\vdots\\
a_{n}
\end{matrix}\right)$ where $a_{i}$ is a $m$-dimensional row vector and $M_{\mathbf{F}}\cdot M_{\mathbf{F}}^{T}=(t_{ij})_{n\times n}$.\\
($\Longrightarrow$): Since $\mathbf{F}$ is a covering of $U$, $a_{i}\neq [0, \cdots, 0]$ for all $i\in\{1, \cdots, m\}$.
Hence $t_{ii}=\vee_{k=1}^{m}(a_{ik}\wedge a_{ki})=\vee_{k=1}^{m}(a_{ik}\wedge a_{ik})=1$.\\
($\Longleftarrow$): If $\mathbf{F}$ is not a covering of $U$, then we suppose $x_{i}\in U-\bigcup\mathbf{F}$.
Thus $a_{i}=[0, \cdots, 0]$ implies $t_{ii}=0$, which is contradictory that all the elements on the main diagonal of $M_{\mathbf{F}}\cdot M_{\mathbf{F}}^{T}$ are one.
\end{proof}

According to the above property of type-1 characteristic matrix of a covering, the relationship between coverings and reflexive relations is established.

\begin{corollary}
Let $\mathbf{F}$ be a family of subsets of $U$ and $\emptyset\notin\mathbf{F}$.
$\mathbf{F}$ is a covering iff there exists a reflexive relation $R$ such that $M_{\mathbf{F}}\cdot M_{\mathbf{F}}^{T}$ is the relational matrix of $R$.
\end{corollary}

According to the definition of the relational matrix of a relation, the type-1 characteristic matrix of a covering is a square boolean matrix, which induces a unique binary relation.
The following proposition represents the relation through covering blocks.

\begin{proposition}
\label{Prelationandset}
Let $\mathbf{C}$ be a covering of $U$ and $R_{\mathbf{C}}$ the relation induced by $\Gamma(\mathbf{C})$.
For all $x_{i}, x_{j}\in U$, $(x_{i}, x_{j})\in R_{\mathbf{C}}$ iff there exists $C\in\mathbf{C}$ such that $x_{i}, x_{j}\in C$.
\end{proposition}
\begin{proof}
Denote $M_{\mathbf{F}}=\left(
\begin{matrix}
a_{1}\\
\vdots\\
a_{n}
\end{matrix}\right)$, where $a_{i}$ is a $m$-dimensional row vector.
Denote $T_{\mathbf{F}}=(t_{ij})_{n\times n}=M_{\mathbf{F}}\cdot M_{\mathbf{F}}^{T}=(t_{ij})_{n\times n}$.\\
For all $x_{i}, x_{j}\in U$, $(x_{i}, x_{j})\in R_{\mathbf{C}}$\\
$\Leftrightarrow t_{ij}=\vee_{k=1}^{m}(a_{ik}\wedge a_{kj})=\vee_{k=1}^{m}(a_{ik}\wedge a_{jk})=1$\\
$\Leftrightarrow a_{i}\wedge a_{j}\neq [0, \cdots, 0]$\\
$\Leftrightarrow$ there exists $C\in\mathbf{C}$ such that $x_{i}, x_{j}\in C$.
\end{proof}

It is interesting that the type-1 characteristic matrix of a covering is the relational matrix of the relation induced by indiscernible neighborhoods of the covering.

\begin{theorem}
\label{T:Discernibleneiborhoodandtype1}
Let $\mathbf{C}$ be a covering of $U$ and $R_{\mathbf{C}}$ the relation induced by $\Gamma(\mathbf{C})$.
For all $x, y\in U$, $(x, y)\in R_{\mathbf{C}}$ iff $y\in I_{\mathbf{C}}(x)=\bigcup\{K\in\mathbf{C}|x\in K\}$.
\end{theorem}
\begin{proof}
($\Longrightarrow$): According to Proposition~\ref{Prelationandset}, if $(x, y)\in R_{\mathbf{C}}$, then there exists $C\in\mathbf{C}$ such that $x, y\in C$.
Hence $y\in C\subseteq\bigcup\{K\in\mathbf{C}|x\in K\}=I_{\mathbf{C}}(x)$.\\
($\Longleftarrow$): If $y\in I_{\mathbf{C}}(x)=\bigcup\{K\in\mathbf{C}|x\in K\}$, then there exists $C\in\mathbf{C}$ such that $x\in C$ and $y\in C$, which implies $(x, y)\in R_{\mathbf{C}}$.
\end{proof}

As we know, equivalence relations and partitions are determined by each other.
Therefore, a question arises: what is the relationship between the type-1 characteristic matrix of a partition and the relational matrix of its corresponding equivalence relation?

\begin{corollary}
Let $R$ be an equivalence relation on $U$ and $M_{R}$ the relational matrix of $R$.
Then $\Gamma(U/R)=M_{R}$.
\end{corollary}
\begin{proof}
We only need to prove $R_{U/R}=R$.
In fact, it is straightforward since $I_{U/R}(x)=\bigcup\{K\in U/R|x\in K\}=[x]_{R}=\{y\in U|(x, y)\in R\}$.
\end{proof}

The above corollary shows that the type-1 characteristic matrix of a partition coincides with the relational matrix of its corresponding equivalence relation.
The following corollary considers another question: which covering blocks removed have no effect on the type-1 characteristic matrix.

\begin{corollary}
Let $\mathbf{C}$ be a covering of $U$ and $K\in\mathbf{C}$.
If there exists $K'\in\mathbf{C}-\{K\}$ such that $K\subseteq K'$, then $\Gamma(\mathbf{C})=\Gamma(\mathbf{C}-\{K\})$.
\end{corollary}

The above corollary indicates that those smaller covering blocks removed have no effect on the type-1 characteristic matrix.

\subsection{Type-2 characteristic matrix of covering}
\label{S:matrixrepresentingofneighborhoodbasedroughset}

The matrix representation of a covering is a framework to study covering-based rough sets and it inherits essential information of the covering.
This subsection constructs a new operation between boolean matrices to study covering-based rough sets in this framework.

\begin{definition}
Let $A = (a_{ij})_{n\times m}$ and $B = (b_{ij})_{m\times p}$ be two boolean matrices.
We define $C = A\odot B$ as follows: $C = (c_{ij})_{n\times p}$,
\begin{center}
$c_{ij}=\wedge_{k=1}^{m}(b_{kj} - a_{ik} + 1)$.
\end{center}
\end{definition}

It is worth noting that $A$ and $B$ are boolean matrices, however $A\odot B$ may not be a boolean matrix.
The following counterexample indicates this argument.

\begin{example}
Suppose $A = \underbrace{(0, \cdots, 0)}_{k}$ and $B = \underbrace{(1, \cdots, 1)}_{k}$, then $A\odot B^{T} = (2)$.
\end{example}

It is interesting that the new operation of any matrix representing a covering and its transpose is a boolean matrix.

\begin{proposition}
Let $\mathbf{C}$ be a covering of $U$ and $M_{\mathbf{C}}$ a matrix representing $\mathbf{C}$.
Then $M_{\mathbf{C}}\odot M_{\mathbf{C}}^{T}$ is a boolean matrix.
\end{proposition}
\begin{proof}
It is straightforward that $c_{ij}\in\{0, 1, 2\}$.
We need to prove $c_{ij}\neq 2$ for all $i, j\in\{1, \cdots, n\}$.\\
$M_{\mathbf{C}}\odot M_{\mathbf{C}}^{T}=\left(
\begin{matrix}
a_{1}\\
\vdots\\
a_{n}
\end{matrix}\right)\odot(a_{1}^{T}, \cdots, a_{n}^{T})=\left(
\begin{matrix}
a_{1}\odot a_{1}^{T} & \cdots & a_{1}\odot a_{n}^{T}\\
\vdots &&\vdots\\
a_{n}\odot a_{1}^{T} & \cdots & a_{n}\odot a_{n}^{T}
\end{matrix}\right) \triangleq (t_{ij})_{n\times n}$.
$t_{ij}=a_{i}\odot a_{j}^{T}=\wedge_{k=1}^{m}(a_{kj} - a_{ik} + 1)=\wedge_{k=1}^{m}(a_{jk} - a_{ik} + 1)$.
If $t_{ij}=2$, then $\wedge_{k=1}^{m}(a_{jk} - a_{ik} + 1)=2$, which implies $a_{jk}=1$ and $a_{ik}=0$ for all $k\in\{1, \cdots, m\}$.
In other words, $x_{i}\notin C_{k}$ for all $k\in\{1, \cdots, m\}$, which is contradictory that $\mathbf{C}$ is a covering of $U$.
Therefore, $t_{ij}\in\{0, 1\}$, i.e., $M_{\mathbf{C}}\odot M_{\mathbf{C}}^{T}$ is a boolean matrix.
\end{proof}

The following proposition points out that the new operation of any matrix representing a covering and its transpose is the same once an order of elements of the universe is given.

\begin{proposition}
Let $\mathbf{C}$ be a covering of $U$ and $M_{1}$, $M_{2}$ matrices representing $\mathbf{C}$.
Then $M_{1}\odot M_{1}^{T}=M_{2}\odot M_{2}^{T}$.
\end{proposition}

\begin{definition}(Type-2 characteristic matrix of covering)
Let $\mathbf{C}$ be a covering of $U$.
Then $M_{\mathbf{C}}\odot M_{\mathbf{C}}^{T}$ is called type-2 characteristic matrix of $\mathbf{C}$, denoted as $\Pi(\mathbf{C})$.
\end{definition}

The following definition introduces an approach to generating a relation from a covering.
There is a close relationship between this relation and neighborhood-based rough sets.

\begin{definition}(Relation induced by a covering~\cite{Zhu09RelationshipBetween})
Let $\mathbf{C}$ be a covering of $U$.
One defines the relation $R(\mathbf{C})$ on $U$ as follows: for all $x, y\in U$,
\begin{center}
$(x, y)\in R(\mathbf{C})\Longleftrightarrow y\in N_{\mathbf{C}}(x)$.
\end{center}
\end{definition}

The type-2 characteristic matrix of a covering is the relational matrix of the relation induced by neighborhoods.

\begin{theorem}
\label{T:relexiveandtransitiverelationwithbooleanmatrix}
Let $\mathbf{C}$ be a covering of $U$.
Then $\Pi(\mathbf{C})$ is the relational matrix of $R(\mathbf{C})$.
\end{theorem}
\begin{proof}
Denote $M_{\mathbf{C}}=\left(
\begin{matrix}
a_{1}\\
\vdots\\
a_{n}
\end{matrix}\right)$ and $\Pi(\mathbf{C})=M_{\mathbf{C}}\odot M_{\mathbf{C}}^{T}=(t_{ij})_{n\times n}$.
If $t_{ij}=1$, then $\wedge_{k=1}^{m}(a_{jk}-a_{ik}+1)=1$, which implies if $a_{ik}=1$, then $a_{jk}=1$.
In other words, if $x_{i}\in C_{k}$, then $x_{j}\in C_{k}$.
Hence $x_{j}\in\bigcap\{K\in\mathbf{C}|x_{i}\in K\}=N_{\mathbf{C}}(x_{i})$, i.e., $(x_{i}, x_{j})\in R(\mathbf{C})$.
If $t_{ij}=0$, then $\wedge_{k=1}^{m}(a_{jk}-a_{ik}+1)=0$, which implies that there exists $k_{0}\in\{1, \cdots, m\}$ such that $a_{jk_{0}}=0$ and $a_{ik_{0}}=1$.
In other words, $x_{i}\in C_{k_{0}}$ and $x_{j}\notin C_{k_{0}}$.
Thus $x_{j}\notin C_{k_{0}}\supseteq \bigcap\{K\in\mathbf{C}|x_{i}\in K\}=N_{\mathbf{C}}(x_{i})$, which implies $x_{j}\notin N_{\mathbf{C}}(x_{i})$, i.e., $(x_{i}, x_{j})\notin R(\mathbf{C})$.
This completes the proof.
\end{proof}

The following proposition considers a question: which covering blocks removed have no effect on the type-2 characteristic matrix.

\begin{proposition}
Let $\mathbf{C}$ be a covering of $U$ and $K\in\mathbf{C}$.
If $K$ is reducible, then $\Pi(\mathbf{C})=\Pi(\mathbf{C}-\{K\})$.
\end{proposition}
\begin{proof}
It is straightforward since $N_{\mathbf{C}}(x)=N_{\mathbf{C}-\{K\}}(x)$ for all $x\in U$ if $K$ is reducible in $\mathbf{C}$.
\end{proof}

The above proposition presents that reducible elements of a covering removed have no effect type-2 characteristic matrix.
It is natural that a covering and its reduct have the same type-2 characteristic matrix.

\begin{proposition}
Let $\mathbf{C}$ be a covering of $U$.
Then $\Pi(\mathbf{C})=\Pi(Reduct(\mathbf{C}))$.
\end{proposition}

The fifth covering upper approximation operator can be concisely represented by boolean matrices.
$\chi_{_Y}$ is used to denote the characteristic function of $Y$ in $U$; in other words, for all $y\in U$, $\chi_{_Y}(y)=1$ if and only if $y\in Y$.

\begin{theorem}
Let $\mathbf{C}$ be a covering of $U$.
Then for all $X\subseteq U$,
\begin{center}
$\chi_{_{IH(X)}}=\Pi(\mathbf{C})\cdot \chi_{_X}$.
\end{center}
\end{theorem}
\begin{proof}
Denote $M_{\mathbf{C}}\odot M_{\mathbf{C}}^{T}=(t_{ij})_{n\times n}$.
If $X = \emptyset$, then $\chi_{_{IH(X)}}=M_{\mathbf{C}}\odot M_{\mathbf{C}}^{T}\cdot [0, \cdots, 0]^{T}=[0, \cdots, 0]^{T}$, which implies $IH(X)=\emptyset$.\\
$x_{i}\in IH(X)$\\
$\Leftrightarrow\chi_{_{IH(X)}}(x_{i})=1$\\
$\Leftrightarrow \vee_{k=1}^{m}(t_{ik}\wedge \chi_{_X}(x_{k}))=1$\\
$\Leftrightarrow \exists k_{0}\in\{1, \cdots, m\}$ such that $t_{ik_{0}}=\chi_{_X}(x_{k_{0}})=1$\\
$\Leftrightarrow x_{k_{0}}\in N_{\mathbf{C}}(x_{i})$, $x_{k_{0}}\in X$\\
$\Leftrightarrow N_{\mathbf{C}}(x_{i})\bigcap X\neq\emptyset$.
\end{proof}

\begin{example}
\label{E:upperapproximation}
Let $U = \{a, b, c, d, e, f\}$ and $\mathbf{C} = \{K_{1}, K_{2}, K_{3}, K_{4}\}$ where $K_{1} = \{a, b\}$, $K_{2} = \{a, c, d\}$, $K_{3} = \{a, b, c, d\}$ and $K_{4} = \{d, e, f\}$.
Then \\$\Pi(\mathbf{C})=M_{\mathbf{C}}\odot M_{\mathbf{C}}^{T}=\left(
\begin{matrix}
1 & 1 & 1 & 0\\
1 & 0 & 1 & 0\\
0 & 1 & 1 & 0\\
0 & 1 & 1 & 1\\
0 & 0 & 0 & 1\\
0 & 0 & 0 & 1
\end{matrix}\right)\odot\left(
\begin{matrix}
1 & 1 & 0 & 0 & 0 & 0\\
1 & 0 & 1 & 1 & 0 & 0\\
1 & 1 & 1 & 1 & 0 & 0\\
0 & 0 & 0 & 1 & 1 & 1
\end{matrix}\right)=\left(
\begin{matrix}
1 & 0 & 0 & 0 & 0 & 0\\
1 & 1 & 0 & 0 & 0 & 0\\
1 & 0 & 1 & 1 & 0 & 0\\
0 & 0 & 0 & 1 & 0 & 0\\
0 & 0 & 0 & 1 & 1 & 1\\
0 & 0 & 0 & 1 & 1 & 1
\end{matrix}\right).$
\begin{center}
\begin{tabular}{|c|c|c|c|}
  \hline
  $X$ & $\chi_{_X}$ & $\Pi(\mathbf{C})\cdot \chi_{_X}$ & $IH(X)$\\
  \hline
  $\{a\}$             & $[1~0~0~0~0~0]^{T}$ & $[1~1~1~0~0~0]^{T}$ & $\{a, b, c\}$ \\
  $\{b, c\}$          & $[0~1~1~0~0~0]^{T}$ & $[0~1~1~0~0~0]^{T}$ & $\{b, c\}$ \\
  $\{a, d, e\}$       & $[1~0~0~1~1~0]^{T}$ & $[1~1~1~1~1~1]^{T}$ & $\{a, b, c, d, e, f\}$ \\
  $\{b, d, e, f\}$    & $[0~1~0~1~1~1]^{T}$ & $[0~1~1~1~1~1]^{T}$ & $\{b, c, d, e, f\}$ \\
  \hline
\end{tabular}
\end{center}
\end{example}

Similarly, the fifth covering lower approximation operator is also represented by the type-2 characteristic matrix.

\begin{theorem}
Let $\mathbf{C}$ be a covering of $U$.
Then for all $X\subseteq U$ and $X\neq\emptyset$,
\begin{center}
$\chi_{_{IL(X)}}=\Pi(\mathbf{C})\odot \chi_{_X}$.
\end{center}
\end{theorem}
\begin{proof}
Denote $M_{\mathbf{C}}\odot M_{\mathbf{C}}^{T}=(t_{ij})_{n\times n}$.\\
$x_{i}\in IL(X)\Leftrightarrow\chi_{_{IL(X)}}(x_{i})=1$\\
$\Leftrightarrow \vee_{k=1}^{m}(\chi_{_X}(x_{k}) - t_{ik} + 1)=1$\\
$\Leftrightarrow$ if $t_{ik}=1$, then $\chi_{_X}(x_{k})=1 $\\
$\Leftrightarrow$ if $x_{k}\in N_{\mathbf{C}}(x_{i})$, then $x_{k}\in X$\\
$\Leftrightarrow N_{\mathbf{C}}(x_{i})\subseteq X$.
\end{proof}

\begin{example}
As shown in Example~\ref{E:upperapproximation}, the following table is obtained:
\begin{center}
\begin{tabular}{|c|c|c|c|}
  \hline
  $X$ & $\chi_{_X}$ & $\Pi(\mathbf{C})\odot \chi_{_X}$ & $IL(X)$\\
  \hline
  $\{a\}$             & $[1~0~0~0~0~0]^{T}$ & $[1~0~0~0~0~0]^{T}$ & $\{a\}$ \\
  $\{b, c\}$          & $[0~1~1~0~0~0]^{T}$ & $[0~0~0~0~0~0]^{T}$ & $\emptyset$ \\
  $\{a, d, e\}$       & $[1~0~0~1~1~0]^{T}$ & $[1~0~0~1~0~0]^{T}$ & $\{a, d\}$ \\
  $\{b, d, e, f\}$    & $[0~1~0~1~1~1]^{T}$ & $[0~0~0~1~1~1]^{T}$ & $\{d, e, f\}$ \\
  $\{a, b, c, d, e\}$ & $[1~1~1~1~1~0]^{T}$ & $[1~1~1~1~0~0]^{T}$ & $\{a, b, c, d\}$\\
  \hline
\end{tabular}
\end{center}
\end{example}

Through the new defined operation and the type-1 characteristic matrix of a covering, the second covering upper and lower approximation operators are concisely equivalently characterized.

\begin{theorem}
\label{T:secondupperapproximationrepreentedusingbooleanmatrix}
Let $\mathbf{C}$ be a covering of $U$. Then for all $X\subseteq U$,
\begin{center}
$\chi_{_{SH(X)}}=\Gamma(\mathbf{C})\cdot \chi_{_X}$.
\end{center}
\end{theorem}
\begin{proof}
We need to prove $\{x\in U|I(x)\bigcap X\neq\emptyset\}=SH(X)$ for all $X\subseteq U$, where $I(x)=\bigcup\{K\in\mathbf{C}|x\in K\}$.
For all $x\in SH(X)$, then there exists $K\in\mathbf{C}$ such that $K\bigcap X\neq\emptyset$.
Then $K\bigcap X\subseteq I(x)\bigcap X\neq\emptyset$, which implies $x\in \{x\in U|I(x)\bigcap X\neq\emptyset\}$.
Conversely, for all $x\notin\{x\in U|I(x)\bigcap X\neq\emptyset\}$, i.e., $I(x)\bigcap X=\emptyset$ which implies $I(x)\subseteq X^{c}$.
Since $K\bigcap X\subseteq I(x)\bigcap X$ for all $x\in K$, $K\bigcap X=\emptyset$.
Hence $x\notin SH(X)$.
This completes the proof.
\end{proof}

\begin{example}
As shown in Example~\ref{E:upperapproximation}, the following table is presented:
\begin{center}
\begin{tabular}{|c|c|c|c|}
  \hline
  $X$ & $\chi_{_X}$ & $\Gamma(\mathbf{C})\cdot \chi_{_X}$ & $SH(X)$\\
  \hline
  $\{a\}$             & $[1~0~0~0~0~0]^{T}$ & $[1~1~1~1~0~0]^{T}$ & $\{a, b, c, d\}$ \\
  $\{b, c\}$          & $[0~1~1~0~0~0]^{T}$ & $[1~1~1~1~0~0]^{T}$ & $\{a, b, c, d\}$ \\
  $\{a, d, e\}$       & $[1~0~0~1~1~0]^{T}$ & $[1~1~1~1~1~1]^{T}$ & $\{a, b, c, d, e, f\}$ \\
  $\{a, b, c, d\}$    & $[1~1~1~1~0~0]^{T}$ & $[1~1~1~1~1~1]^{T}$ & $\{a, b, c, d, e, f\}$ \\
  $\{a, b, d, e, f\}$ & $[1~1~0~1~1~1]^{T}$ & $[1~1~1~1~1~1]^{T}$ & $\{a, b, c, d, e, f\}$\\
  \hline
\end{tabular}
\end{center}
\end{example}


\begin{theorem}
\label{T:theseondlowerapproximationusingbooleanmatrix}
Let $\mathbf{C}$ be a covering of $U$. Then for all $X\subseteq U$,
\begin{center}
$\chi_{_{SL(X)}}=\Gamma(\mathbf{C})\odot\chi_{_X}$.
\end{center}
\end{theorem}
\begin{proof}
We need to prove $\{x\in U|I(x)\subseteq X\}=SL(X)$ for all $X\subseteq U$.
In fact, it is straightforward.
\end{proof}

\begin{example}
As shown in Example~\ref{E:upperapproximation}, the following table is exhibited:
\begin{center}
\begin{tabular}{|c|c|c|c|}
  \hline
  $X$ & $\chi_{_X}$ & $\Gamma(\mathbf{C})\odot\chi_{_X}$ & $SL(X)$\\
  \hline
  $\{a\}$             & $[1~0~0~0~0~0]^{T}$ & $[0~0~0~0~0~0]^{T}$ & $\emptyset$ \\
  $\{a, b\}$          & $[1~1~0~0~0~0]^{T}$ & $[0~0~0~0~0~0]^{T}$ & $\emptyset$ \\
  $\{d, e, f\}$       & $[0~0~0~1~1~1]^{T}$ & $[0~0~0~0~1~1]^{T}$ & $\{e, f\}$ \\
  $\{a, b, c, d\}$    & $[1~1~1~1~0~0]^{T}$ & $[1~1~1~0~0~0]^{T}$ & $\{a, b, c\}$ \\
  $\{a, b, d, e, f\}$ & $[1~1~0~1~1~1]^{T}$ & $[0~0~0~0~1~1]^{T}$ & $\{e, f\}$\\
  \hline
\end{tabular}
\end{center}
\end{example}

\begin{theorem}
\label{T:axiomatizationofthesixthupperapproximation}
Let $\mathbf{C}$ be a covering of $U$. Then for all $X\subseteq U$,
\begin{center}
$\chi_{_{XH(X)}}=\Pi(\mathbf{C})^{T}\cdot\Pi(\mathbf{C})\cdot \chi_{_X}$, \\
~~$\chi_{_{XL(X)}}=\Pi(\mathbf{C})^{T}\cdot\Pi(\mathbf{C})\odot \chi_{_X}$.
\end{center}
\end{theorem}
\begin{proof}
For a covering $\mathbf{C}$ of $U$, we construct a special covering $Cov(\mathbf{C}) = \{N(x)|x\in U\}$ induced by $\mathbf{C}$.
Then $XH_{\mathbf{C}}(X) = SH_{Cov(\mathbf{C})}(X)$ and $XL_{\mathbf{C}}(X) = SL_{Cov(\mathbf{C})}(X)$ for all $X\subseteq U$.
Since $\Pi(\mathbf{C})^{T}$ is a matrix representing $Cov(\mathbf{C})$, according to Theorems~\ref{T:secondupperapproximationrepreentedusingbooleanmatrix} and~\ref{T:theseondlowerapproximationusingbooleanmatrix}, $\chi_{_{XH(X)}}=\Pi(\mathbf{C})^{T}\cdot\Pi(\mathbf{C})\cdot \chi_{_X}$ and $\chi_{_{XL(X)}}=\Pi(\mathbf{C})^{T}\cdot\Pi(\mathbf{C})\odot \chi_{_X}$.
\end{proof}

\begin{example}
As shown in Example~\ref{E:upperapproximation}, the following two tables are revealed:
\begin{center}
\begin{tabular}{|c|c|c|c|}
  \hline
  $X$                 &$\chi_{_X}$           & $\Pi(\mathbf{C})^{T}\cdot\Pi(\mathbf{C})\cdot\chi_{_X}$   & $XH(X)$\\
  \hline
  $\{a\}$             & $[1~0~0~0~0~0]^{T}$ & $[1~1~1~1~0~0]^{T}$                 & $\{a, b, c, d\}$          \\
  $\{a, b\}$          & $[1~1~0~0~0~0]^{T}$ & $[1~1~1~1~0~0]^{T}$                 & $\{a, b, c, d\}$          \\
  $\{a, b, c\}$       & $[1~1~1~0~0~0]^{T}$ & $[1~1~1~1~0~0]^{T}$                 & $\{a, b, c, d\}$ \\
  $\{d, e, f\}$       & $[0~0~0~1~1~1]^{T}$ & $[1~0~1~1~1~1]^{T}$                 & $\{a, c, d, e, f\}$       \\
  $\{a, d, e, f\}$    & $[1~0~0~1~1~1]^{T}$ & $[1~1~1~1~1~1]^{T}$                 & $\{a, b, c, d, e, f\}$ \\
  \hline
\end{tabular}
\end{center}
\begin{center}
\begin{tabular}{|c|c|c|c|}
  \hline
  $X$                 &$\chi_{_X}$           & $\Pi(\mathbf{C})^{T}\cdot\Pi(\mathbf{C})\odot \chi_{_X}$   & $XL(X)$\\
  \hline
  $\{a\}$             & $[1~0~0~0~0~0]^{T}$ & $[0~0~0~0~0~0]^{T}$                 & $\emptyset$          \\
  $\{a, b\}$          & $[1~1~0~0~0~0]^{T}$ & $[0~1~0~0~0~0]^{T}$                 & $\{b\}$          \\
  $\{a, b, c\}$       & $[1~1~1~0~0~0]^{T}$ & $[0~1~0~0~0~0]^{T}$                 & $\{b\}$ \\
  $\{a, b, c, d\}$    & $[1~1~1~1~0~0]^{T}$ & $[1~1~1~0~0~0]^{T}$                 & $\{a, b, c\}$       \\
  $\{a, b, d, e, f\}$ & $[1~1~0~1~1~1]^{T}$ & $[0~1~0~0~1~1]^{T}$                 & $\{b, e, f\}$ \\
  $\{a, b, c, d, e, f\}$ & $[1~1~1~1~1~1]^{T}$ & $[1~1~1~1~1~1]^{T}$              & $\{a, b, c, d, e, f\}$ \\
  \hline
\end{tabular}
\end{center}
\end{example}

Matrix representations of covering approximation operators provide efficient computational models for covering-based rough sets.

\section{Boolean matrix decomposition through covering-based rough sets}
\label{S:booleanmatrixdecompositionusingcoveringbasedroughsets}

In this section, we present a sufficient and necessary condition for a boolean matrix to decompose into the boolean product of another boolean matrix and its transpose, i.e., $B = A\cdot A^{T}$, where $B\in \{0, 1\}^{n\times n}$ and $A\in\{0, 1\}^{n\times m}$.
Here $\{0, 1\}^{n\times m}$ denotes the family of all boolean matrices $C = (C_{ij})_{n\times m}$.
Furthermore $A\cdot A^{T}$ is called an optimal decomposition of $B$ if $A\cdot A^{T} = B$ and $A$ has minimal columns.
Based on the condition, we also consider an optimal boolean matrix decomposition.

\begin{proposition}
Let $\mathbf{C}$ be a covering of $U$ and $M_{\mathbf{C}}$ a matrix representing $\mathbf{C}$.
Then $M_{\mathbf{C}}\cdot M_{\mathbf{C}}^{T}$ is symmetric and $(M_{\mathbf{C}}\cdot M_{\mathbf{C}}^{T})_{ii} = 1$.
\end{proposition}

For all $A, B, C\in\{0, 1\}^{n\times m}$, if $A_{ij} = B_{ij} \vee C_{ij}$ for $i\in\{1, \cdots, n\}$ and $j\in\{1, \cdots, m\}$, then we denote $A = B\vee C$ and we say that $A$ is the union of $B$ and $C$.
If $A_{ij}\leq B_{ij}$ for $i\in\{1, \cdots, n\}$ and $j\in\{1, \cdots, m\}$, then we denote $A \leq B$.
Obviously, $A = B$ if and only if $A \leq B$ and $B \leq A$.

In the following proposition, we break the characteristic matrix of a covering into the union of some ``small" characteristic matrices.

\begin{proposition}
\label{P:subformulaanddecomposition}
Let $\mathbf{C}$ be a covering of $U = \{x_{1}, \cdots, x_{n}\}$. Then
\begin{center}
$M_{\mathbf{C}}\cdot M_{\mathbf{C}}^{T} = \vee_{K\in\mathbf{C}}(M_{\{K\}}\cdot M_{\{K\}}^{T})$.
\end{center}
\end{proposition}
\begin{proof}
It is straightforward that $\vee_{K\in\mathbf{C}}(M_{\{K\}}\cdot M_{\{K\}}^{T}) \leq M_{\mathbf{C}}\cdot M_{\mathbf{C}}^{T}$.
If $(M_{\mathbf{C}}\cdot M_{\mathbf{C}}^{T})_{ij} = 1$, then according to Proposition~\ref{Prelationandset}, $x_{i}\in I_{\mathbf{C}}(x_{j})$.
Hence there exists $K\in\mathbf{C}$ such that $x_{i}, x_{j}\in K$, which implies $(M_{\{K\}}\cdot M_{\{K\}}^{T})_{ij} = 1$.
This proves that $M_{\mathbf{C}}\cdot M_{\mathbf{C}}^{T} \leq \vee_{K\in\mathbf{C}}(M_{\{K\}}\cdot M_{\{K\}}^{T})$.
To sum up, we prove $M_{\mathbf{C}}\cdot M_{\mathbf{C}}^{T} = \vee_{K\in\mathbf{C}}(M_{\{K\}}\cdot M_{\{K\}}^{T})$.
\end{proof}

In fact, the characteristic matrix of a covering can be represented by the characteristic functions of covering blocks.

\begin{corollary}
Let $\mathbf{C}$ be a covering of $U$. Then
\begin{center}
$M_{\mathbf{C}}\cdot M_{\mathbf{C}}^{T} = \vee_{K\in\mathbf{C}}(\chi_{_{K}}\cdot \chi_{_{K}}^{T})$.
\end{center}
\end{corollary}

Inspired by the ``small" characteristic matrix of a covering block, we define the sub-formula, which serves as a foundation for designing an algorithm for an optimal boolean matrix decomposition.

\begin{definition}
Let $U = \{x_{1}, \cdots, x_{n}\}$ be a universe and $A\in\{0, 1\}^{n\times n}$.
$A$ is called a sub-formula on $U$ if there exists $X \subseteq U$ such that $A = M_{\{X\}}\cdot M_{\{X\}}^{T}$.
\end{definition}

\begin{example}
Let $A_{1}=\left[\begin{array}{lcr}
  1 & 0 & 1 \\
  0 & 0 & 0\\
  1 & 0 & 1 \\
\end{array}\right]$ and $A_{2}=\left[\begin{array}{lcr}
  1 & 1 & 1 \\
  1 & 1 & 1\\
  1 & 1 & 1 \\
\end{array}\right]$.
Then $A_{1}$ and $A_{2}$ are sub-formulas on $U = \{x_{1}, x_{2}, x_{3}\}$ since there exist $X_{1} = \{x_{1}, x_{3}\}$ and $X_{2} = \{x_{1}, x_{2}, x_{3}\}$ such that $A_{1} = M_{\{X_{1}\}}\cdot M_{\{X_{1}\}}^{T}$ and $A_{2} = M_{\{X_{2}\}}\cdot M_{\{X_{2}\}}^{T}$.
\end{example}

In the following theorem, we present a sufficient and necessary condition for a square boolean matrix to decompose into the boolean product of another boolean matrix and its transpose.

\begin{theorem}
\label{T:equivalentcharacterizationforbooleandecomposition}
Let $B\in\{0, 1\}^{n\times n}$.
Then there exists $A \in\{0, 1\}^{n\times m}$ such that $B = A\cdot A^{T}$ iff $(B = B^{T})\wedge (\forall i, j\in\{1, \cdots, n\}, B_{ij} = 1 \Rightarrow B_{ii} = 1)$.
\end{theorem}
\begin{proof}
($\Longrightarrow$): On one hand, $B^{T} = (A\cdot A^{T})^{T} = (A^{T})^{T}\cdot A^{T} = A\cdot A^{T} = B$. On the other hand, we suppose $A = \left[\begin{array}{lcr}
  a_{1}^{T}\\
  \vdots\\
  a_{n}^{T}\\
\end{array}\right]$, where $a_{i}$ is a $m$-dimensional column vector.
$\forall i, j\in\{1, \cdots, n\}$, if $B_{ij} = 1$, then $a_{i}^{T}\cdot a_{j} = 1$, i.e., there exists $h\in\{1, \cdots, m\}$ such that $a_{ih} = a_{jh} = 1$.
Hence $B_{ii} = a_{i}^{T}\cdot a_{i} = 1$ and $B_{jj} = a_{j}^{T}\cdot a_{j} = 1$.

($\Longleftarrow$): Suppose that $A_{1}, \cdots, A_{m}$ is all the sub-formulas that satisfy $A_{i}\leq B$, i.e., $A_{i}\vee B = B$, $i\in \{1, \cdots, k\}$.
If $(B = B^{T})\wedge (\forall i, j\in\{1, \cdots, n\}, B_{ij} = 1 \Rightarrow B_{ii} = 1)$, then $B = \vee_{i}^{m}A_{i}$.
Since $A_{i}$ is a sub-formula on $U = \{x_{1}, \cdots, x_{n}\}$, there exists a unique $K_{i}$ such that $A_{i} = M_{\{K_{i}\}}\cdot M_{\{K_{i}\}}^{T}$.
Hence $B = \vee_{i}^{m}A_{i} = \vee_{i}^{m}(M_{\{K_{i}\}}\cdot M_{\{K_{i}\}}^{T})$.
According to Proposition~\ref{P:subformulaanddecomposition}, $B = M_{\mathbf{C}}\cdot M_{\mathbf{C}}^{T}$ where $\mathbf{C} = \{K_{1}, \cdots, K_{m}\}$, which implies $B = A\cdot A^{T}$ where $A = M_{\mathbf{C}}$.
\end{proof}

In applications, $min\{k|B = A\cdot A^{T}, A\in\{0, 1\}^{n\times k}\}$ has some special meaning.
For instance, in role mining, it represents the minimal number of roles~\cite{LuVaidyaAtluri08Optimal}.

Inspired by the reducible element~\cite{YaoYao12Covering,ZhuWang03Reduction}, we define the notion of general intersection-reducible element, which is similar to the relative covering reduct proposed in literature~\cite{DuHuZhuMa11Rule}.

\begin{definition}
Let $\mathbf{C}$ be a covering of $U$ and $K\in\mathbf{C}$.
If there exists $K'\in\mathbf{C}$ such that $K\subseteq K'$, then $K$ is called a general intersection-reducible element of $\mathbf{C}$; otherwise, it is called general intersection-irreducible.
\end{definition}

Note that the above definition can be extended to any family of subsets of a set.
Hence in the rest of this paper, this notion will be used to deal with the redundancy with respect to a family of subsets of a set.

In fact, the family of all general intersection-reducible elements of $\mathbf{C}$ is unique, and we call it the general intersection-reduct of $\mathbf{C}$, and denote it as $GIR(\mathbf{C})$.
The following theorem explores the relationship between an optimal boolean matrix decomposition and the general intersection-reduct.

\begin{theorem}
\label{T:thoughtsofalgorithms}
Let $B\in \{0, 1\}^{n\times n}$ where $B = B^{T}$ and $\forall i, j\in\{1, \cdots, n\}, B_{ij} = 1 \Rightarrow B_{ii} = 1$.
Then $B = M_{GIR(\mathbf{C}_{B})}\cdot M_{GIR(\mathbf{C}_{B})}^{T}$ is an optimal decomposition of $B$, where $\mathbf{C}_{B} = \{K|\exists A$ s.t. $(A\leq B)\wedge(M_{\{K\}}\cdot M_{\{K\}}^{T} = A)\} - \{\emptyset\}$.
\end{theorem}

Theorem~\ref{T:thoughtsofalgorithms} shows that finding an optimal boolean matrix decomposition is equivalent to find the general intersection-reduct of a covering.
An example is provided to illustrate this interesting transformation.

\begin{example}
Let $B = \left[\begin{matrix}
  1 & 1 & 0 & 1 & 0 \\
  1 & 1 & 0 & 1 & 0 \\
  0 & 0 & 1 & 1 & 0 \\
  1 & 1 & 1 & 1 & 0 \\
  0 & 0 & 0 & 0 & 0
\end{matrix}\right]$.
Then the maximal sub-formulas containing in $B$ are $A_{1} = \left[\begin{matrix}
  1 & 1 & 0 & 1 & 0 \\
  1 & 1 & 0 & 1 & 0 \\
  0 & 0 & 0 & 0 & 0 \\
  1 & 1 & 0 & 1 & 0 \\
  0 & 0 & 0 & 0 & 0
\end{matrix}\right]$ and $A_{2} = \left[\begin{matrix}
  0 & 0 & 0 & 0 & 0 \\
  0 & 0 & 0 & 0 & 0 \\
  0 & 0 & 1 & 1 & 0 \\
  0 & 0 & 1 & 1 & 0 \\
  0 & 0 & 0 & 0 & 0
\end{matrix}\right]$.
We suppose $U = \{x_{1}, \cdots, x_{5}\}$ and $K_{1} = \{x_{1}, x_{2}, x_{4}\}$ and $K_{2} = \{x_{3}, x_{4}\}$, then $A_{1} = M_{\{K_{1}\}}\cdot M_{\{K_{1}\}}^{T}$ and $A_{2} = M_{\{K_{2}\}}\cdot M_{\{K_{2}\}}^{T}$.
Hence $B = A_{1}\vee A_{2} = (M_{\{K_{1}\}}\cdot M_{\{K_{1}\}}^{T}) \vee (M_{\{K_{2}\}}\cdot M_{\{K_{2}\}}^{T}) = M_{\mathbf{C}}\cdot M_{\mathbf{C}}^{T}$ where $\mathbf{C} = \{K_{1}, K_{2}\}$.
\end{example}

The following algorithm shows how to obtain an optimal boolean matrix decomposition using covering-based rough sets.

\begin{algorithm}[H]
\caption{An algorithm for optimal boolean matrix decomposition}
\label{Analgorithmforminimalbooleanmatrixdecomposition}
\textbf{Input}: $B\in\{0, 1\}^{n\times n}$ \\
\textbf{Output}: $A\in\{0, 1\}^{n\times m}$ with the minimal column
\begin{algorithmic}[1]
\STATE Denote $U = \{x_{1}, \cdots, x_{n}\}$;
\STATE Compute all maximal sub-formulas containing in $B$, and denote as $A_{1}, \cdots, A_{m}$;
\IF {$B = \vee_{i=1}^{m} A_{i}$}
   \STATE Compute $K_{i}$ such that $A_{i} = M_{\{K_{i}\}}\cdot M_{\{K_{i}\}}^{T}$ for $i\in\{1, \cdots, m\}$;
   \STATE Return $A = [\chi_{_{K_{1}}}, \cdots, \chi_{_{K_{m}}}]$;
   \STATE // $A\cdot A^{T}$ is an optimal boolean matrix decomposition of $B$;
\ELSE
\STATE Return $\emptyset$;  // there does not exist $A\in\{0, 1\}^{n\times m}$ such that $B = A\cdot A^{T}$;
\ENDIF
\end{algorithmic}
\end{algorithm}

\section{An application of boolean matrix decomposition to axiomatization of covering approximation operators}
\label{S:Axiomatizationofcovering-basedroughsetsusingmatrices}

Axiomatization of covering-based rough sets has attracted much research interest~\cite{Liu08Axiomatic,ZhangLiWu10OnAxiomatic,ZhuWang12TheFourth}.
However, those works are mainly conducted from the viewpoints of set theory and operator theory.
Based on boolean matrix decomposition, we axiomatize three types of approximation operators of covering-based rough sets.
Because of the duality of these three types of covering upper and lower approximation, we consider only their corresponding upper ones.

Let $U = \{x_{1}, \cdots, x_{n}\}$ and an operator $f: 2^{U} \longrightarrow 2^{U}$.
We denote $A_{f} = [\chi_{_f(e_{1})}, \cdots, \chi_{_f(e_{n})}]^{T}$ where $e_{i} = \underbrace{[\cdots, 0, 1, 0, \cdots]}_{i-th}$$^{T}$.

\begin{theorem}
\label{T:thesecondupperapproximationandaxiomatization}
Let $H: 2^{U}\longrightarrow 2^{U}$ be an operator.
Then there exists a covering $\mathbf{C}$ such that $H = SH_{\mathbf{C}}$ iff $A_{H}^{T} = A_{H}$ and $(A_{H})_{ii} = 1$ for all $i\in\{1, \cdots, n\}$.
\end{theorem}
\begin{proof}
($\Longrightarrow$): Since there exists a covering $\mathbf{C}$ such that $H = SH_{\mathbf{C}}$, $A_{H} = \Gamma(\mathbf{C}) = M_{\mathbf{C}}\cdot M_{\mathbf{C}}^{T}$.
According to Proposition~\ref{P:maindiagonalequaltoone}, $(A_{H})_{ii} = 1$ for all $i\in \{1, \cdots, n\}$.
Additionally, $A_{H}^{T} = A_{H}$ is straightforward. \\
($\Longleftarrow$): Since $A_{H}^{T} = A_{H}$ and $(A_{H})_{ii} = 1$ for all $i\in\{1, \cdots, n\}$, according to Theorem~\ref{T:equivalentcharacterizationforbooleandecomposition}, there exists $B\in \{0, 1\}^{n\times m}$ such that $A_{H} = B\cdot B^{T}$.
Specifically, we suppose $B$ is the minimal decomposition of $A_{H}$ and $B = [B_{1}, \cdots, B_{m}]$, then we construct a covering $\mathbf{C} = {K_{1}, \cdots, K_{m}}$ satisfying $B_{i} = \chi_{_{K_{i}}}$ for all $i\in\{1, \cdots, m\}$.
Hence it is straightforward that $H = SH_{\mathbf{C}}$.
\end{proof}

Theorem~\ref{T:thesecondupperapproximationandaxiomatization} shows a sufficient and necessary condition for an operator to be the second upper approximation operator with respect to a covering using boolean matrices.
The following corollary reveals the close connection between covering-based rough sets and generalized rough sets based on relations.

\begin{corollary}
\label{C:thesecondandreflexiveandsymmetricrelation}
Let $H: 2^{U}\longrightarrow 2^{U}$ be an operator.
Then there exists a covering $\mathbf{C}$ such that $H = SH_{\mathbf{C}}$ iff $A_{H}$ is the relational matrix of a reflexive and symmetric relation. \end{corollary}

Corollary~\ref{C:thesecondandreflexiveandsymmetricrelation} presents that the second upper approximation operator with respect to a covering is determined by a reflexive and symmetric relation.
The following theorem explores the relationship between the fifth upper approximation operator of covering-based rough sets and generalized rough sets based on relations.

\begin{theorem}
\label{T:fifthupperapproximationandbooleanmatrix}
Let $H: 2^{U}\longrightarrow 2^{U}$ be an operator.
Then there exists a covering $\mathbf{C}$ such that $H = IH_{\mathbf{C}}$ iff $(A_{H})^{2} = A_{H}$ and $(A_{H})_{ii} = 1$ for all $i\in\{1, \cdots, n\}$.
\end{theorem}

Theorem~\ref{T:fifthupperapproximationandbooleanmatrix} exhibits a sufficient and necessary condition for an operator to be the fifth upper approximation one with respect to a covering using boolean matrices.

\begin{corollary}
\label{C:fifthupperapproximationandreflexiveandtransitiverelation}
Let $H: 2^{U}\longrightarrow 2^{U}$ be an operator.
Then there exists a covering $\mathbf{C}$ such that $H = IH_{\mathbf{C}}$ iff $A_{H}$ is the relational matrix of a reflexive and transitive relation.
\end{corollary}

Corollary~\ref{C:fifthupperapproximationandreflexiveandtransitiverelation} points out that the fifth upper approximation operator with respect to a covering is determined by a reflexive and transitive relation.

\begin{theorem}
\label{T:sixthapproximationoperatorandbooleanmatrix}
Let $H: 2^{U}\longrightarrow 2^{U}$ be an operator.
Then there exists a covering $\mathbf{C}$ such that $H = XH_{\mathbf{C}}$ iff there exists $B\in \{0, 1\}^{n\times n}$ such that $A_{H} = B\cdot B^{T}$ where $B^{2} = B$ and $(B)_{ii} = 1$ for all $i\in\{1, \cdots, n\}$.
\end{theorem}
\begin{proof}
According to Theorem~\ref{T:axiomatizationofthesixthupperapproximation}, it is straightforward.
\end{proof}

Theorem~\ref{T:sixthapproximationoperatorandbooleanmatrix} presents an axiom of the sixth upper approximation operator with respect to a covering from the viewpoint of boolean matrices.

\begin{corollary}
\label{C:sixthupperapproximationopeatroandreflexiveandtransitiverelation}
Let $H: 2^{U}\longrightarrow 2^{U}$ be an operator.
Then there exists a covering $\mathbf{C}$ such that $H = XH_{\mathbf{C}}$ iff there exists a reflexive and transitive relation where $B$ is its relational matrix such that $A_{H} = B\cdot B^{T}$.
\end{corollary}

Corollary~\ref{C:sixthupperapproximationopeatroandreflexiveandtransitiverelation} indicates that the sixth upper approximation operator with respect to a covering is corresponded to a reflexive and transitive relation.

\section{Conclusions}
\label{S:Conclusions}
In this paper, three types of covering lower and upper approximation operators are represented in a boolean matrix form, and then they are applied to boolean matrix decomposition.
Through two types of characteristic matrices of a covering, matrix representations of covering approximation operators are obtained, and then
a sufficient and necessary condition for decomposing a boolean matrix into the boolean product of another boolean matrix and its transpose is provided.
We also design an algorithm for this boolean matrix decomposition.
Moreover, we axiomatize these three types of covering approximation operators through boolean matrix decomposition.
In a word, this work points out an interesting new view to investigate covering-based rough set theory and explore its application.

\section{Acknowledgments}
This work is supported in part by the National Natural Science Foundation of China under Grant No. 61170128,  the Natural Science Foundation of
Fujian Province, China, under Grant Nos. 2011J01374 and 2012J01294, and the Science and Technology Key Project of Fujian Province, China, under
Grant No. 2012H0043.


\end{document}